\documentclass[10pt,twocolumn,letterpaper]{article}

\usepackage{cvpr}              

\usepackage{graphicx}
\usepackage{amsmath}
\usepackage{amssymb}
\usepackage{booktabs}
\usepackage{amsthm}
\usepackage{multirow}
\newtheorem{theorem}{Theorem}[section]
\newtheorem{corollary}{Corollary}[theorem]
\newtheorem{lemma}[theorem]{Lemma}

\newtheorem{assumption}{Assumption}

\DeclareMathOperator*{\argmin}{arg\,min}
\usepackage[pagebackref,breaklinks,colorlinks]{hyperref}

\usepackage[capitalize]{cleveref}
\crefname{section}{Sec.}{Secs.}
\Crefname{section}{Section}{Sections}
\Crefname{table}{Table}{Tables}
\crefname{table}{Tab.}{Tabs.}

\def\cvprPaperID{10621} 
\def\confName{CVPR}
\def\confYear{2022}

\begin{document}

\title{Topology Preserving Local Road Network Estimation from Single Onboard Camera Image}

\author{ Yigit Baran Can\textsuperscript{1}\space\space\space\space Alexander Liniger\textsuperscript{1}\space\space\space\space Danda Pani Paudel\textsuperscript{1}\space\space\space\space Luc Van Gool\textsuperscript{1,2}\\
\textsuperscript{1}Computer Vision Lab, ETH Zurich\space\space\space\space \textsuperscript{2}VISICS, ESAT/PSI, KU Leuven \\ {\tt\small $\{$yigit.can, alex.liniger, paudel, vangool$\}$@vision.ee.ethz.ch} }
\maketitle

\begin{abstract}
Knowledge of the road network topology is crucial for autonomous planning and navigation. Yet, recovering such topology from a single image has only been explored in part. Furthermore, it needs to refer to the ground plane, where also the driving actions are taken. This paper aims at extracting the local road network topology, directly in the bird's-eye-view (BEV), all in a complex urban setting. The only input consists of a single onboard, forward looking camera image.
We represent the road topology using a set of directed lane curves and their interactions, which are captured using their intersection points. 
To better capture topology, we introduce the concept of \emph{minimal cycles} and their covers. A minimal cycle is the smallest cycle formed by the directed curve segments (between two intersections). The cover is a set of curves whose segments are involved in forming a minimal cycle. We first show that the covers suffice to uniquely represent the road topology. The covers are then used to supervise  deep neural networks, along with the lane curve supervision. These learn to predict the road topology from a single input image. The results on the NuScenes and Argoverse benchmarks are significantly better than those obtained with baselines. Code: \url{https://github.com/ybarancan/TopologicalLaneGraph}. 
\end{abstract}

\section{Introduction}

How would you give directions to a driver? One of the most intuitive ways is by stating turns, instead of distances. For example, taking \emph{the third right turn} is more intuitive and robust than going \emph{straight for 100 meters and turn right}. This observation motivates us to model road networks using the involved lanes and their intersections. We model the lane intersections ordered in the direction of traffic. Given a reference centerline $L$ and two lines $I_1, I_2$ intersecting $L$, we consider the set of all possible lines intersecting $L$ between the intersection points $L-I_1$ and $L-I_2$ as an equivalence class. Such modelling allows us to supervise the learning process explicitly using the topological structure of the road network. In turn, the topological consistency during inference is improved. Consider a car moving from the green point P upwards in Fig~\ref{fig:rebut}, which needs to take the first left turn. In the two estimates of \cite{Can_2021_ICCV}, the first left leads to different lanes. While the underlying directed graphs have the same connectivity in all estimates, they have very different topological structures which play an important role in decision making.

\begin{figure}
    \centering
    \includegraphics[width=\linewidth]{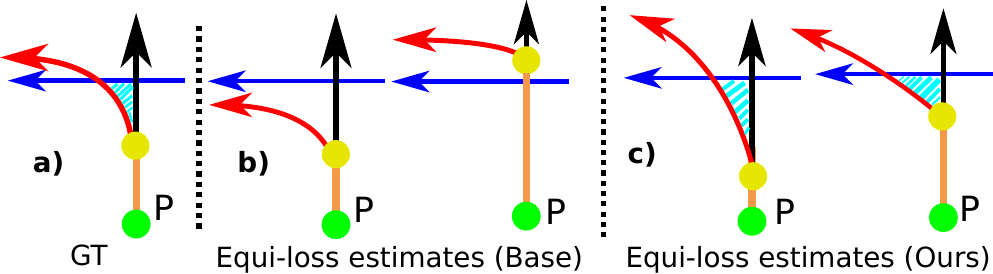}
    \caption{
    GT (a), two estimates of base method \cite{Can_2021_ICCV} with similar loss(b), and two estimates of Ours with similar loss (c). Yellow dots refer to connections. Our formulation encourages the shaded (in cyan) region to exist in the prediction, which in turn, ensures preservation of order of intersections. 
    }
    \label{fig:rebut}
    \vspace{-1em}
\end{figure}

For autonomous driving, the information contained in the local road network surrounding the car is vital for the decision making of the autonomous system. The local road network is both used to predict the motion of other agents \cite{cui2019multimodal, hong2019rules, rella2021decoder, zaech2020action} as well as to plan ego-motion \cite{DBLP:conf/rss/BansalKO19, chen2020learning}. The most popular approach to represent the road network is in terms of lane graph based HD-maps, which contain both the information about the centerlines and their connectivity. Most existing methods address the problem of road network extraction by using offline generated HD-maps in combination with a modular perception stack~\cite{jaritz20202d,seif2016autonomous,ma2019exploiting,ravi2018real,casas2021mp3}. However, offline HD-maps based solutions have two major issues: (i)
dependency on the precise localization in the HD-map~\cite{ma2019exploiting,yang2018hdnet}, (ii) requirement to construct and maintain such maps. These requirements severely limit the scalability of autonomous driving to operate in geographically restricted areas. To avoid offline mapping \cite{Can_2021_ICCV} proposed to directly estimate the local road network online from just one onboard image. Inspired by this approach and given the importance of topological consistency for graph based maps we propose to directly supervise the map generation network to estimate topological consistent road networks. 

\begin{figure*}
    \centering
    \includegraphics[width=0.8\linewidth]{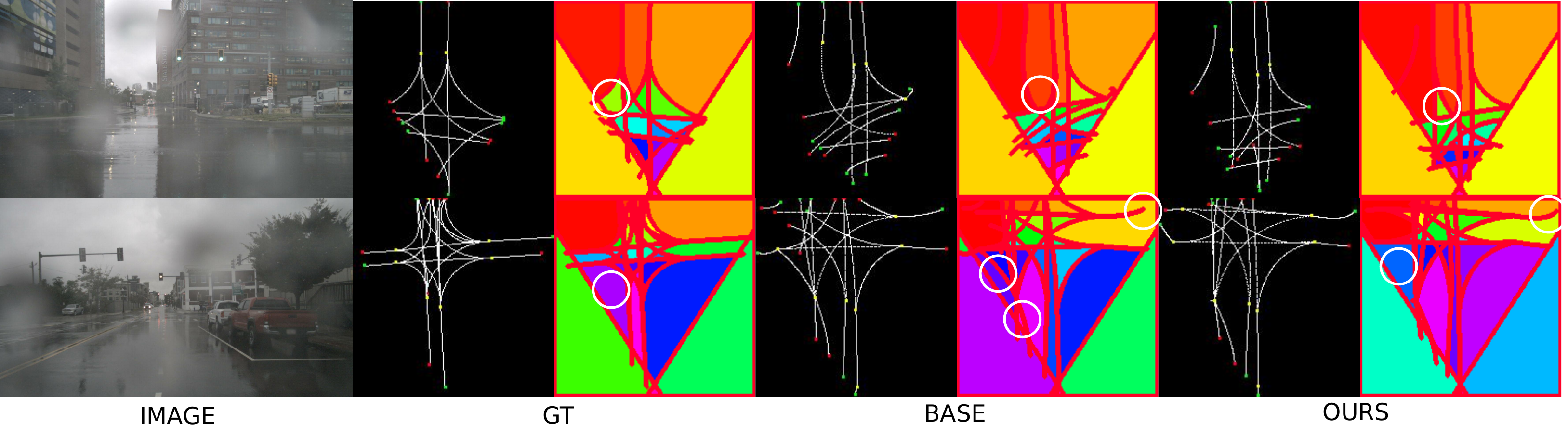}
    \vspace{-0.5em}
    \caption{The road networks and the resulting closed regions (minimal cycles, shown in different colors). In the road network, traffic flows from green to red dots and the yellow dots are connection points of two centerlines. Proposed formulation learns to preserve the identities of the centerlines enclosing the minimal cycles.
    Some interesting regions are shown by white circles on colored images.}
    \label{fig:matching1}
\end{figure*}

Starting from \cite{Can_2021_ICCV}, we represent the local road network using a set of Bezier curves. Each curve represents a driving lane, which is directed along the traffic flow with the help of starting and end points. However, compared to \cite{Can_2021_ICCV} we also consider the topology of the road network, which is modelled by these directed curves and their intersections. More precisely, we rely on the following statement: \emph{if the intersection order of every lane, with the other lanes, remains the same, the topology of the road network is preserved}. Directly estimating the order of intersections is a hard problem.
Therefore, we introduce the concept of \emph{minimal cycles} and their covers. A minimal cycle is the smallest cycle formed by the directed curve segments (between two intersections). In Fig.~\ref{fig:matching1}, each minimal cycle is represented by  different colors. The cover is the set of curves whose segments are involved in forming a minimal cycle. Given two sets of curves (one estimation and one ground truth), if all curves in both sets intersect in the same order, then the topology defined by both sets are equivalent. In this work, we show that such equivalence can be measured by comparing the covers of the minimal cycles. Based on our findings, we supervise deep neural networks that learn to predict topology preserving road networks from a single image. To our knowledge, this is the first work studying the problem of estimating topology of a road network, thus going beyond the traditional lane graphs.

Our model predicts lane curves and how they connect as well as the minimal cycles with their covers. During the learning process, both curves and cycles are jointly supervised. The curve supervision is performed by matching predicted curves to those of ground truth by means of the Hungarian algorithm. Similarly, the cycle supervision matches the cover of the predicted cycles to those of the ground truth. Such joint supervision encourages our model to predict accurate and topology preserving road networks during inference, without requiring the branch for cycle supervision.    

For online mapping in autonomous driving and other robotic applications, it is crucial to directly predict the road network in the Birds-Eye-View (BEV) since the action space of the autonomous vehicles is the ground. This stands in contrast to traditional scene understanding which mainly takes place in the image plane. It was recently shown that performing BEV scene understanding in the image plane, and then project it to the ground plane is inferior to directly predicting the output on the ground plane~\cite{DBLP:conf/cvpr/RoddickC20,can2020understanding,wu2020motionnet,paz2020probabilistic,yang2018hdnet,casas2018intentnet,philion2020lift}. Compared to our approach these methods do not provide the local road network, but a segmentation on the BEV. Note that the road regions alone do not provide the desired topological information. Similar methods that perform lane detection are limited to highway like roads where the topology of the lanes is trivial, and are not able to predict road networks in urban scenarios with intersections which is the setting we are interested. As said \cite{Can_2021_ICCV} is able to predict such road networks but does not consider the topology, and we will show that our topological reasoning can improve the methods proposed in \cite{Can_2021_ICCV}.


Our predicted connected curves provide us a full lane graph HD map. To this end, our major contributions can be summarized as follows.
\begin{enumerate}
\setlength{\itemsep}{0pt}
\setlength{\parskip}{0pt}
\item We propose a novel formulation for the topology of a road network, which is complementary to the classic lane graph approach 
\item We show that a neural network can be trained end-to-end to produce topologically accurate lane graphs from a single onboard image
\item We propose novel metrics to evaluate the topological accuracy of an estimated lane graph
\item The results obtained by our method are superior to the compared methods in both traditional and topological structure metrics

\end{enumerate}

\section{Related Works}
Existing works can be roughly grouped in two distinct groups; first, offline methods, that extract HD map style road networks from aerial images or aggregated sensor data. Second, online methods, which either estimate lane boundaries or perform semantic understanding on the BEV plane, only given current onboard sensor information. Our method is located between the two approaches, estimating HD map style lane graphs, however, based on onboard monocular images. 



\noindent\textbf{Road network extraction:} Early works on road network extraction use aerial images~\cite{auclair1999survey, richards1999remote}. Building upon the same setup, recent works~\cite{batra2019improved,sun2019leveraging,ventura2018iterative} perform the network extraction more effectively. Aerial imaging-based approaches only provide coarse road networks. Such predictions may be useful for routing, however, they are not accurate enough for action planning. 

\noindent\textbf{High definition maps:}  HD maps are often reconstructed offline using aggregated 2D and 3D visual information~\cite{liang2019convolutional, homayounfar2018hierarchical,liang2018end}. Although these works are the major motivation behind our work, they require dense 3D point clouds for accurate HD map reconstruction. These methods are also offline methods which recover HD maps in some canonical frame. 

The usage of the recovered maps requires an accurate localization, in many cases. A similar work to ours is~\cite{DBLP:conf/cvpr/HomayounfarMLU18}, where the lane boundaries are detected on highways in the form of polylines. An extension of~\cite{DBLP:conf/cvpr/HomayounfarMLU18} uses a RNN to generate initial boundary points in 3D point clouds. These initial points are then used as seeds for a Polygon-RNN \cite{DBLP:conf/cvpr/AcunaLKF18} that predicts lane boundaries. Our method differs from~\cite{DBLP:conf/cvpr/HomayounfarMLU18} in: (i) point clouds vs. single image input, (ii) highway lane boundaries vs. lane centerlines in an unrestricted setting.

\noindent\textbf{Lane estimation:}
There is considerable research in lane estimation using monocular cameras \cite{DBLP:conf/ivs/NevenBGPG18, DBLP:conf/iccvw/GansbekeBNPG19}.
The task is either performed directly on the image plane \cite{DBLP:conf/iccv/GarnettCPLL19, DBLP:journals/corr/abs-2002-06604} or in the BEV plane by projecting the image to the ground plane \cite{DBLP:journals/corr/abs-2011-01535, DBLP:conf/siu/YenIaydinS18,DBLP:conf/ivs/NevenBGPG18}. However, this line of research mainly focuses on highway and country roads, without intersections. In such cases the topology of the resulting road lane work is often trivial since lines to not intersect. Our approach focuses on urban traffic with complex road networks where the topology is fundamental.

\noindent\textbf{BEV understanding:}
Visual scene understanding on the BEV has recently become popular due to its practicality~\cite{DBLP:conf/cvpr/RoddickC20,philion2020lift,can2020understanding}. Some methods also combine images with LIDAR~\cite{pan2020cross,hendy2020fishing}. Maybe the most similar to our method are~\cite{DBLP:conf/cvpr/RoddickC20,DBLP:journals/corr/abs-2002-08394,can2020understanding}, which use a single image or monocular video frames for BEV HD-map semantic understanding. However, these methods do not offer structured outputs, therefore their usage for planning and navigation is limited. 

In summary, in our paper we work in a setting similar to \cite{homayounfar2019dagmapper}, where the output is a directed acyclic graph. However, the input is not an aggregated image and LIDAR data, but just one onboard image. Thus, the same sensor setup as existing lane estimation works, these however are not designed to work in urban environments. In fact our setting is identical to \cite{Can_2021_ICCV}, but our work does focus on the topology of the lane graph and proposes a method to directly supervise the network to estimate topologically correct graphs.

\section{Method}

\subsection{Lane Graph Representation}
Following \cite{Can_2021_ICCV}, we represent the local road network as a directed graph of lane centerline segments which is often called the lane graph. Let this directed graph be $G(V, E)$ where $V$ are the vertices of the graph (the centerlines) and the edges $E \subseteq \{(x,y)\; |\; (x,y) \in V^2\}$ represent the connectivity among those centerlines. The connectivity can be summarized by the incidence matrix $A$ of the graph $G(V, E)$. A centerline $x$ is connected to another centerline $y$, i.e. $(x,y) \in E$ if and only if the centerline $y$'s starting point is the same as the end point of the centerline $x$. This means $A[x,y] = 1$ if the centerlines $x$ and $y$ are connected. 
We represent centerlines with Bezier curves. 



\subsection{Topological Representation}

While the directed graph builds an abstract high level representation of the traffic scene, the graph also introduces fundamental topological properties about the road scene. The topological properties depend on the intersections of the centerlines whereas the lane graph depends on the connectivity of the centerlines\footnote{A connection is defined by the incidence matrix of $G(V,E)$ whereas an intersection between two curves is defined in the geometric sense.}. Thus, also considering the topology gives complementary information which we can use to estimate better representations.

We assume that the target BEV area is a bounded 2D Euclidean space, where the known bounding curves represent the borders of the field-of-view (FOV). Identical to the lane graph, each curve has a direction which represents the flow of traffic, while boundary curves have arbitrary directions. We denote the set of all curves including the border curves as $C$. To establish our later results, we assume that any two curves can intersect at most once and a curve does not intersect with itself. Due to the restricted FOV and relative short curve segments this assumption is not restrictive. Moreover, in a lane graph no curve is floating, since every end of a curve either connects to another curve or leaves the bounded space, which also results in an intersection. Let $c\in C$ be a curve and $I_c$ be the ordered sequence of intersections along the direction of curve $c$ and $I_c(m) \in \textbf{P}$ be the $m^{\text{th}}$ intersection of the sequence, where $\textbf{P}$ is the set of all intersection points. The set of $I_c$ for all curves $c$ is denoted by $I$. Combining the curves $C$ and intersection order $I$ we can form our topological structure $T(C, I)$ that together with $G(V, E)$ define the local road network (see a) of Fig.~\ref{fig:basics}). In this example with linear curves, the order of intersection $I_c$ for all curves is given.   

\begin{figure}
    \centering
    \includegraphics[width=.8\linewidth]{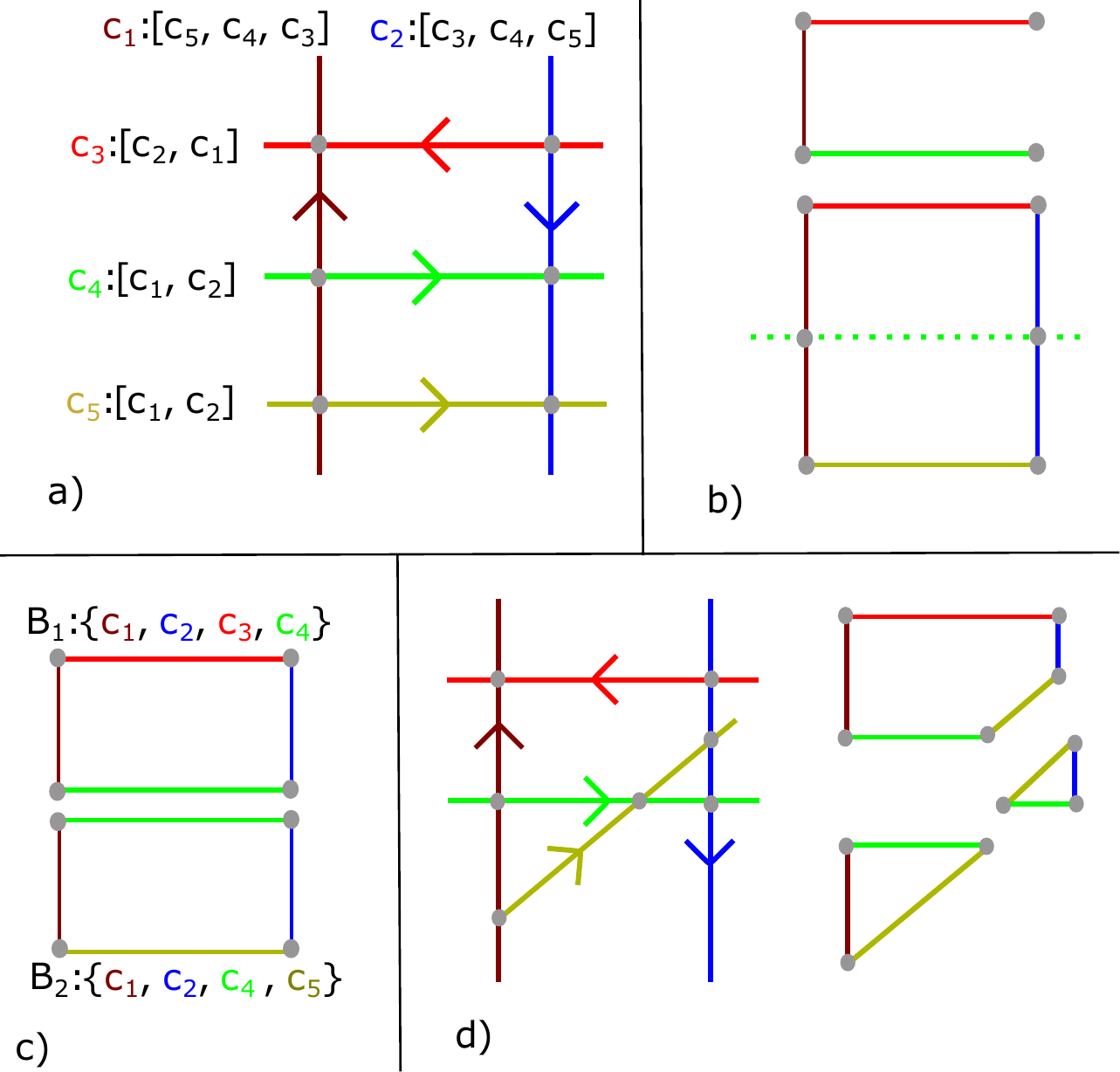}
    \caption{The graphical illustrations of our basic definitions are shown. Gray dots show the intersection points. Every curve part between any two consecutive intersection point is a curve segment. a) The complete network with all the curves $\{c_i\}$ and the order of intersections for each curve. b) An example polycurve (above), a closed (but not minimal) polycurve (below). c) Two minimal closed polycurves (minimal cycles) with the set of curves enclosing them, i.e. their minimal covers $B$. d) Another configuration of the same curves and the resulting minimal cycles.}
    \label{fig:basics}
    \vspace{-1.5em}
\end{figure}

When estimating the lane graph of a traffic scene, or in fact any graph structures formed by curves, we would not only like to correctly estimate the lane graph $G(V,E)$ but also the topological properties $T(C,I)$. However, estimating the ordering of the intersection points directly is very challenging. In the following, we will show that under some assumptions the intersection order $I_c$ for each curve is equivalent to the covers of minimal cycles of curves. This equivalence allows us to efficiently add a topological reasoning to our network. 
Let us first define minimal cycles and covers. A curve segment $S_c(i,j)$ is the subset of the curve $c$ between successive intersection points $i$ and $j$, known due to $I_c$. We define a polycurve $PC$ as a sequence of curve segments $PC_S = (S)_m | S_m(j)=S_{m+1}(i)$. 
A closed polycurve $CC$ is a polycurve with no endpoints, which completely encloses an area (see b) of Fig.~\ref{fig:basics}). A minimal closed polycurve or minimal cycle $MC$ is a closed polycurve where no curve intersects the area enclosed by $MC$, see Fig.~\ref{fig:basics} c). Note that minimal cycles form a partition of the bounded space. Finally, given a polycurve that forms a MC, we can also define the corresponding minimal cover $B$, which is the set union of the curves that the segments in that polycurve belong to, or in other words the list of curves that form the minimum cycle, see Fig~\ref{fig:basics} c) and d).
What makes minimal covers $B$ so interesting is that although they are relatively simple, we will show in the following that they still hold the complete topological information of the road graph and are equivalent to the intersection order $I$.

To establish this equivalence, let us first state the following results that link intersection orders to minimal cycles and covers, which holds under mild conditions detailed in the supplementary material.
\begin{lemma}
\label{lemma1_main}
A minimal closed polycurve (minimal cycle) $MC$ is uniquely identified by its minimal cover $B$. 
\end{lemma}
\begin{proof}
See supplementary material for proof.
\end{proof}
For the statement to be wrong, the same curve $c_i$ of the minimal cover $B$ would need to generate another minimal cycle. Which intuitively becomes hard under the assumption that curves are only allowed to intersect once. For lines as shown in Fig.~\ref{fig:basics} this is not possible. For general curves the proof becomes more involved and needs some further assumptions which can be found in the supplementary.

Given that we have a link between the minimal cover and minimal cycles we now focus on to relationship between the intersection orders $I$ and the minimal cycles. 
\begin{lemma}
\label{lemma2}
Let a set of curves $C_1$ and the induced intersection orders $I_1$ form the structure $T_1=(C_1, I_1)$. Applying any deformations on the curves in $C_1$, excluding removal or addition of curves, results in a new induced intersection order that creates $T_2=(C_2, I_2)$. Given these two typologies, $I_1 = I_2 \iff {MC}_1 = {MC}_2$. In other words, the global intersection order of the two structures are the same if and only if the sets of minimal cycles are the same.
\end{lemma}
\begin{proof}
See supplementary material for proof.
\end{proof}


Given this equivalence between intersection orders and minimum cycle we can state our main result.
\begin{corollary}
\label{coro}
From Lemma \ref{lemma1_main} and Lemma \ref{lemma2}, given a structure $T=(C, I)$, $I$ can be uniquely described by the set of minimal covers $B$.
\end{corollary}
The remarkable fact about Corollary \ref{coro} is that we converted a global ordering problem into a detection problem. Instead of creating a sequence for each curve, it is enough to detect minimal cycles where each minimal cycle can be represented by a one-hot vector of the curves in $T$ which shows whether a curve is in the minimal cover of the particular minimal cycle or not.

\subsection{Structural Mapping}


\begin{figure*}

\centering
\includegraphics[width=.8\linewidth]{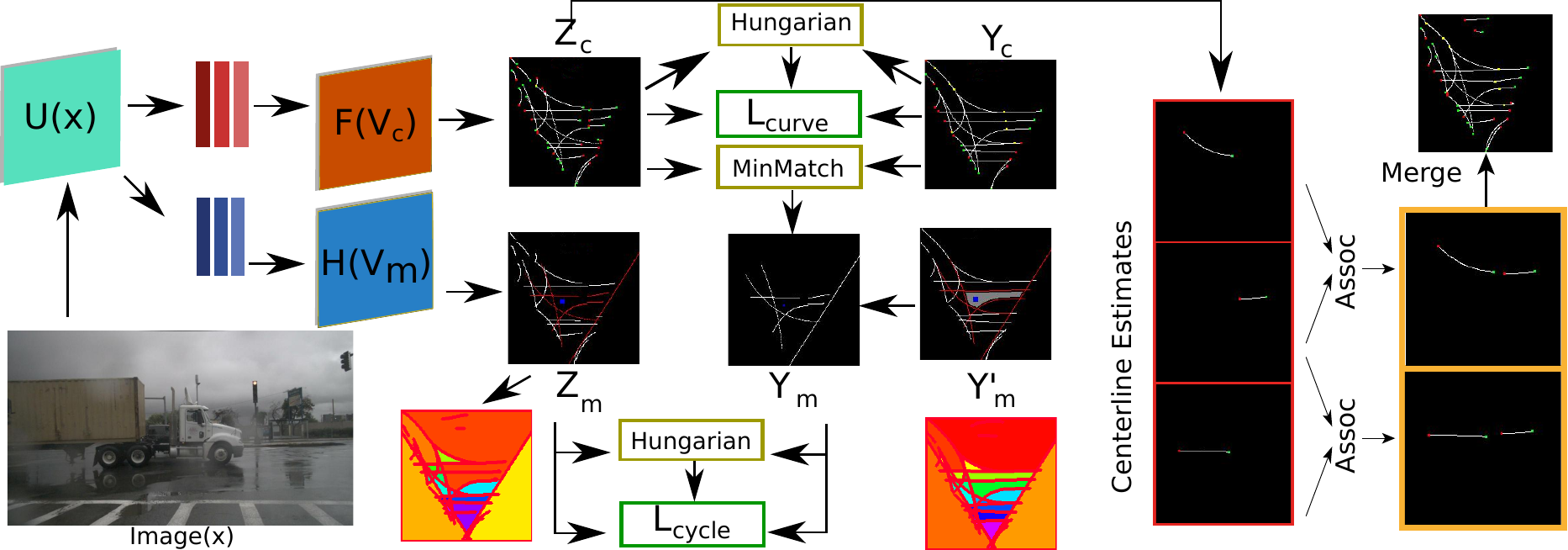}
\vspace{-1.em}
\caption{The training procedure of the proposed framework produces the centerline curve estimates as well as the minimal cycle covers. The matchings between curves and cycles provide a consistent and higher level supervision to the method, resulting in topological understanding of the scene. Connection step is applied to every pair of selected curves to estimate connectivity probability, where the connected centerlines are modified so that their corresponding endpoints coincide.
}
\vspace{-1.em}
\label{fig:matching}
\end{figure*}

The previous theoretical results allow us to train a deep neural network that jointly estimates the curves and the intersection orders. Therefore, we build a mapping between the curves and minimal covers for both the estimated $T_E$ and the ground truth $T_T$ structures.

Using a neural network, we predict a fixed number of curves and minimum cycles, which is larger than the real number of curves and cycles in any scene. Thus, let there be a function $U(x)$ that takes the input $x$ (a camera image in our case) and outputs two matrices, $V_c$ of size $N\times D$ which is a $D$ dimensional embedding for all $N$ curve candidates and $V_m$ of size $M\times E$, which is a $E$ dimensional embedding for all $M$ minimal cycle candidates. The two embedding matrices are each processed by a function, $F(V_c)$ and $H(V_m)$. $F(V_c)$ processes the curve candidate embedding $V_c$ and generates a matrix output $Z_c^q \in \mathbb{R}^{N\times \theta}$ containing the parameters for the $N$ curves and $Z_c^p \in \mathbb{R}^{N}$ the probability that the $i^{\text{th}}$ curve exists. $H(V_m)$ processes the minimum cycle candidate embedding, and generates three outputs each describing a property of minimum cycles. First, $Z_m^q \in \mathbb{R}^{M\times (N+K)}$ the estimated minimal cover for each of the $M$ candidates, describing the probability that one of the $N$ candidate curves and $K$ FOV boundary curves belongs to the cover. Second, $Z_m^p \in \mathbb{R}^{M}$ the probability that a candidate minimal cycle exists and finally, $Z_m^r(a)\in \mathbb{R}^{M\times 2}$ an auxiliary output estimating the centers of the candidate minimal cycles. Thus, our framework generates a set of curve and minimum cycle candidates, see Fig.~\ref{fig:matching} for an illustration.


\subsection{Training Framework}
\label{training}

The output of the network is, (i) a set of candidate curves and (ii) minimal cycles that are defined with respect to the candidate curves. In training we use Hungarian matching on the L1 difference between the control points of centerlines. However, it is more complex for the minimum cycles, where it is fundamental that the matching between the ground truth topology and the estimated topology is consistent. Let there be $N'$ true curves and $M'$ true minimal cycles with $K$ boundary curves. Similarly, $Y{'}_c^q \in \mathbb{R}^{N'\times \theta}$ represents the true curve parameters, $Y{'}_m^q \in \{0,1\}^{M'\times  (N' + K)}$ the minimal covers with respect to the true curves, and $Y_m^r$ the true centers of the minimum cycles.



\paragraph{Min Matching.} Since a ground truth (GT) minimal cycle is defined on GT curves while detected minimal cycles are defined on estimated curves, we must first form a matching between estimated and GT curves. Using Hungarian matching is not ideal since it does not consider the fragmentation of the estimated curves. Fragmentation is the situation when several connected estimated curves represent one GT curve. Therefore, often the estimated candidate minimal cycles will have more candidate curves than their GT counterparts. Due to the one-to-one matching in the Hungarian algorithm, a long GT curve can only be matched to one short, fragmented curve, even though combining the estimated fragments would result in a closer approximation. Thus, we instead match each candidate curve to its closest GT curve. This means every candidate curve is matched to exactly one GT curve, while a GT curve can be matched to any number (including zero) of candidate curves. 



After min matching, we create a new target for minimal cycles estimates that we denote by $Y_m^q \in \{0,1\}^{M'\times (N + K)}$. An entry in $Y_m^q(i,j)$ is 1 if the GT curve to which the $j^{\text{th}}$ estimated curve is matched is in the $i^{\text{th}}$ true minimal cycle. In other words, we set all the matched estimated curves to one if their corresponding true curve is present in a minimal cycle. Given this modified GT minimal cycle label and the estimated minimal cycles, we run Hungarian matching to find the pairs used for the loss calculations. This allows a consistent training of the estimated topology.


\begin{figure*}
    \centering
    \includegraphics[width=.9\linewidth]{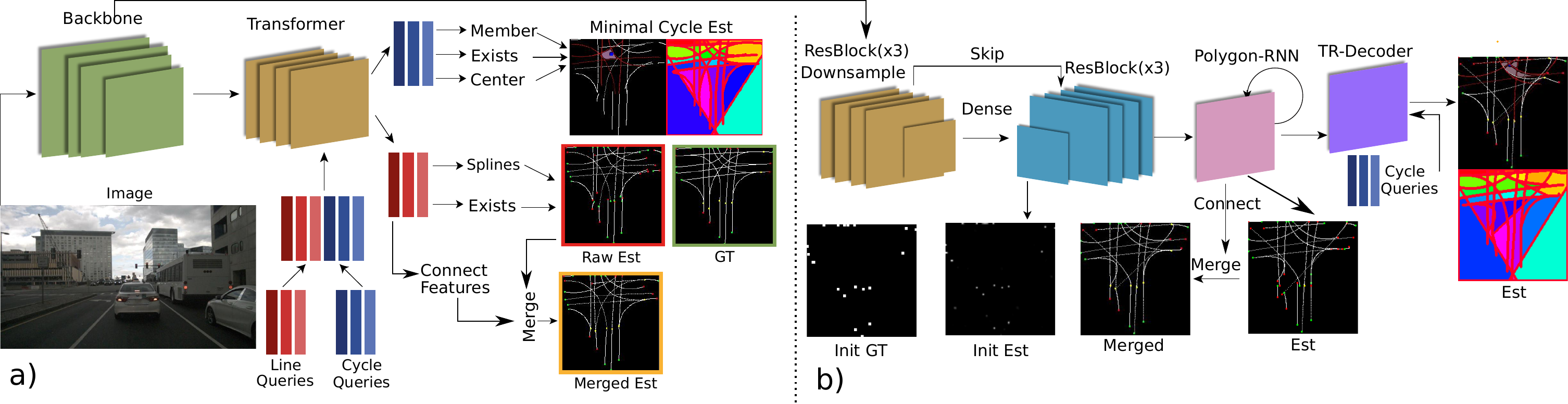}
    \caption{The proposed networks: (a) Minimal cycle transformer (\textbf{Ours/TR}) and (b) Minimal cycle polygon RNN (\textbf{Ours/PRNN}). 
    {Ours/TR} processes two sets of queries (curve and minimal cycle) jointly to produce corresponding feature vectors. These vectors are then fed to MLPs for final estimations.  b) Ours/PRNN  has three parts: 1) Initial point estimation, 2) Polygon-RNN that outputs the subsequent control points of a curve given the initial points, and   3) minimal cycle decoder. 
    }
    \label{fig:transformer}
    \vspace{-1.5em}
\end{figure*}

For the connectivity of the curves, we explicitly estimate the incidence matrix $A$ of $V(G,E)$ in our network. This is done by extracting a feature vector for each candidate centerline and building a classifier $\hat{A}(C_i, C_j)$ that takes two feature vectors belonging to curves $C_i$ and $C_j$ and outputs the probability of their association. The training uses the Hungarian matched curves to establish the correct order. The estimated incidence matrix allows a \textbf{merging} post-processing step during test time, where the endpoints of the curves are modified, so that connected curves coincide.

The centerline spline control points and minimal cycle centers are trained with an $L_1$ loss, while we utilize the binary cross-entropy for centerline and minimal cycle probability. We also use the binary cross-entropy for the membership loss of minimal cycles, i.e. between $Z_m^q$ and $Y_m^q$ and for the connectivity. The total loss then becomes $L = L_{curve} + \alpha L_{cycle}$, where $L_{curve} = L_{splines} + \beta_eL_{exists} + \beta_cL_{connect}$, and $L_{cycle} = L_{member} + \beta_dL_{exists} + \beta_fL_{center}$, with $\alpha$ and $\beta_x$ hyperparameters.


\section{Network Architectures}

Following \cite{Can_2021_ICCV}, we focus on two different architectures to validate the impact of our formulation. The first architecture is based on transformers \cite{DBLP:conf/eccv/CarionMSUKZ20} while the second approach is based on Polygon-RNN \cite{DBLP:conf/cvpr/AcunaLKF18}.

\subsection{Transformer}
We modify the transformer-based architecture proposed in \cite{Can_2021_ICCV}. We use two types of learned query vectors: centerline (curve) and cycle queries. We concatenate centerline and cycle queries before being processed by the transformer. Therefore, curves and cycles are jointly estimated. The transformer outputs the processed queries that correspond to $V_c$ and $V_m$ in our formulation. Finally, we pass these vectors through two-layer MLPs to produce the estimates $Z_c$ and $Z_m$. The overview is given in Fig~\ref{fig:transformer}. Note that the addition of the MC formulation adds negligible parameters since the number of parameters in the transformer is fixed. We call the transformer model with MC, Ours/TR.


As a baseline, we added an RNN on the base transformer to estimate the order of intersections directly and provide supervision to the network. The RNN processes each centerline query output from the transformer independently and generates an $N+K+1$ dimensional vector at each time step that represents the probability distribution of intersecting one of the $N+K$ curves and one `end' token. The RNN is supervised by the true intersection orders converted to estimate centerlines through Hungarian matching. We named this method TR-RNN, see Suppl. Mat. for details. 


\subsection{Polygon-RNN}
The second network is based on Polygon-RNN \cite{DBLP:conf/cvpr/AcunaLKF18} and is similar to \cite{DBLP:conf/cvpr/HomayounfarMLU18}, where the authors generate lane boundaries from point clouds. We adapt \cite{DBLP:conf/cvpr/HomayounfarMLU18} to work with images and to output centerlines rather than lane boundaries. Following \cite{Can_2021_ICCV}, we use a fully connected sub-network that takes $V_c$ as input and outputs a grid. Each element represents the probability of an initial curve point of a curve starting at that location, i.e. $Z_c^p$. 

Given the initial locations and the backbone features, Polygon-RNN \cite{DBLP:conf/cvpr/AcunaLKF18} produces the next control points of the centerline. We fix the number of iterations of Polygon-RNN to the number of spline coefficients used to encode centerlines. The approach described so far forms the base Polygon-RNN. With Polygon-RNN producing $Z_c^q$, we add a transformer decoder to the architecture to detect the minimal cycles. We use a set of minimal cycle queries similar to our transformer architecture, where the queries are processed with final feature maps of Polygon-RNN.
Therefore, in the transformer decoder, the query vectors attend the whole set of estimated centerlines to extract the minimal cycle candidates. For this process, we pad the RNN states to a fixed size and add positional encoding. This ensures that the decoder receives the information regarding the identity of the curves. The processed query vectors are passed to the same MLPs as in the transformer architecture to produce the set of minimal cycle estimates $Z_m$. Fig~\ref{fig:transformer} outlines this approach, which we call Ours/PRNN. Different to the transformer based method, this is a two stage process, where first the centerlines are estimated and then the minimal cycles.

\begin{figure*}
\vspace{-1em}
    \centering
    \includegraphics[width=.9\linewidth]{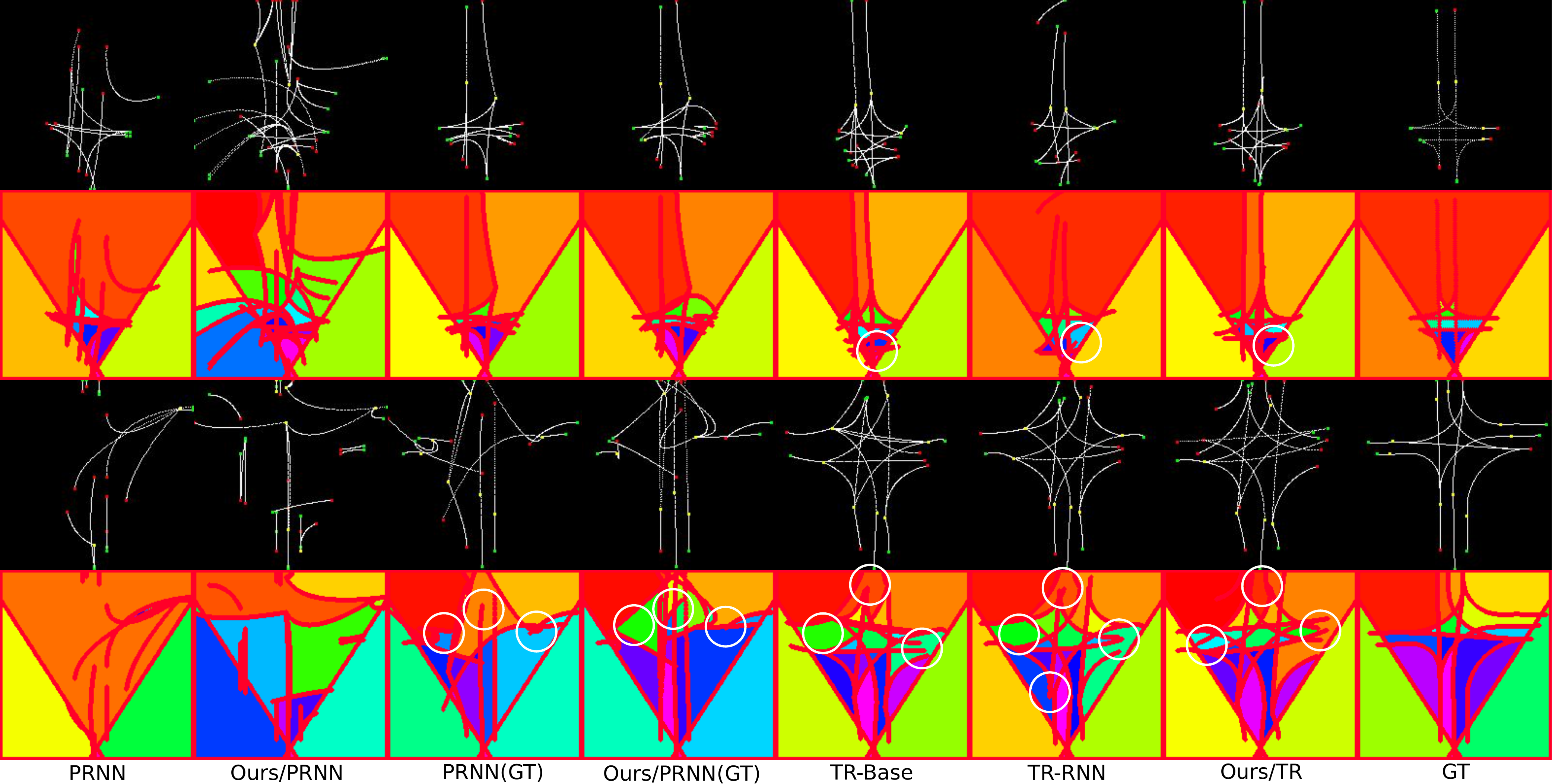}
    \vspace{-1em}
    \caption{Visual results of the different methods compared, with input and GT, on Argoverse (Top) and Nuscenes (Bottom) datasets.}
    \label{fig:visual_compare}
    \vspace{-1em}
    
\end{figure*}

    

    
\section{Metrics}\label{sec:metrices}
Several metrics were proposed in \cite{Can_2021_ICCV} to measure accuracy in estimating the lane graph. They are \emph{M-F-Score}, \emph{Detection ratio}, and \emph{Connectivity}. These metrics do not cover the topology of the road network. Thus, we propose two new metrics that capture the accuracy in estimating topology. The proposed metrics complement the existing lane graph metrics to give a full picture of the accuracy of the estimated road network.

\underline{\emph{Minimal-Cycle Minimal Cover ($B$).}}
We measure the minimal cycle accuracy in 2 inputs. First, minimal cycles are extracted from the estimations. We use the procedure of Section~\ref{training} to obtain $Y_m^q$. 
Then, these cycles are matched using Hungarian matching to calculate true positives, true negatives, and false positives. This metric is referred to as \textbf{MC-F}. 
We also measure the accuracy of the minimal cycle network. Similar to MC-F, first $Y_m^q$ is obtained, and second we threshold $Z_m^p$ to obtain the detected $Z_m^q$. Then, we apply Hungarian matching and calculate statistics on matched cycles. We call this metric \textbf{H-GT-F}, which is applicable only if the minimal cycles are detected. \textbf{H-GT-F} measures the MC-network's performance in estimating the true cycles in the true topology. 
Finally, \textbf{H-EST-F} measures the MC-head's performance in detecting the \emph{estimated} cycles.
Since the extracted MCs and the MC head estimations are with respect to estimated curves, we directly run Hungarian matching on the extracted and estimated MCs. \\
\underline{\emph{Intersection Order ($I$ of $G(C, I)$).}}
To measure the performance of the methods in preserving the intersection order, we start with min-matching. Then for each true curve, we select the closest matched estimate. For a given true curve $C_i$, let the matched curve be $S_i$.
We extract the order of intersections from both $C_i$ and $S_i$ and apply the Levenshtein edit distance between them. The distance is then normalized by the number of intersections of the true curve. We refer to this metric as \textbf{I-Order}.

\section{Experiments}

We use NuScenes \cite{nuscenes2019} and Argoverse \cite{DBLP:conf/cvpr/ChangLSSBHW0LRH19} datasets.  Both datasets provide HD-Maps in the form of centerlines. We convert the world coordinates of the centerlines, to the camera coordinate system of the current frame, then resample these points with the target BEV resolution and discard any point that is outside the region-of-interest (ROI). The points are then normalized with the ROI bounds $[0,1]^2$. We extract the control points of the Bezier curve for this normalized coordinate system. The ground truth and the estimations of the method are also represented in the same coordinate system. We use the same train/val split proposed in \cite{DBLP:conf/cvpr/RoddickC20}.

\noindent\textbf{Implementation.}
We use images of size 448x800 and the target BEV area is from -25 to 25m in x-direction and 1 to 50m in z direction with a 25cm resolution. Due to the limited complexity of the centerlines, three Bezier control points are used. We use two sets of 100 query vectors for centerlines and minimal cycles: one for right (Boston \& Argoverse) and one for left sided traffic (Singapore). The backbone network is Deeplab v3+ \cite{DBLP:conf/eccv/ChenZPSA18} pretrained on the Cityscapes dataset \cite{Cordts2016Cityscapes}. Our implementation is in Pytorch and runs with 11FPS without batching and including all association steps. 
When training Polygon-RNN, we use true initial points for training of the RNN, following \cite{DBLP:conf/cvpr/HomayounfarMLU18}. To train the initial point subnetwork, we use focal loss \cite{DBLP:journals/pami/LinGGHD20}.

\noindent\textbf{Baselines.}
We compare against state-of-the-art~\textbf{transformer} and~\textbf{Polygon-RNN} based methods proposed in \cite{Can_2021_ICCV} as well as another baseline which uses the method~\textbf{PINET}~\cite{DBLP:journals/corr/abs-2002-06604} to extract lane boundaries. The extracted lane boundaries are then projected onto the BEV using the ground truth transformation. We then couple pairs of lane boundaries and extract the centerlines using splines. This baseline is not evaluated for connectivity.

\section{Results}

\begin{figure*}
    \centering
    \includegraphics[width=.9\linewidth]{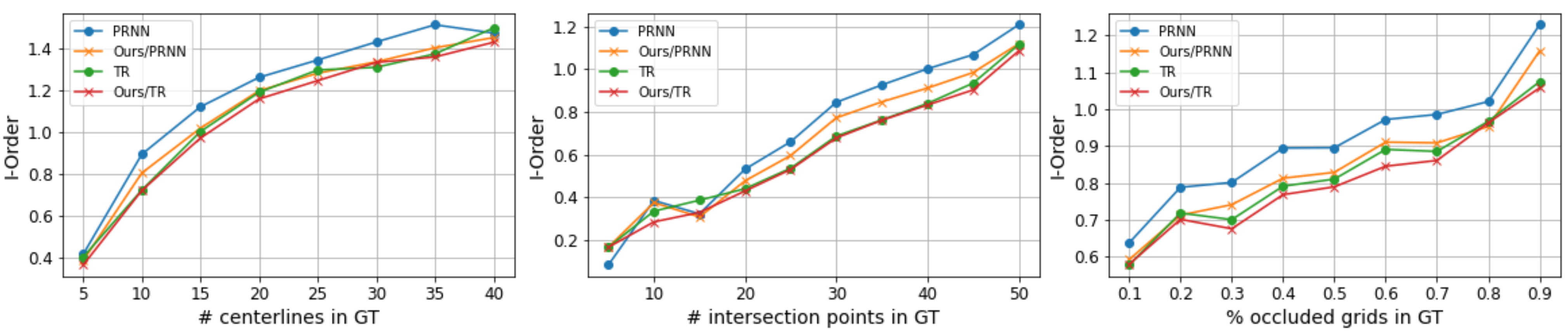}
    \caption{The I-Order measures (lower is better) with number of centerlines, intersection points, and occlusion in scene. The x-axis measures reflect the complexity of scenes for the visual understanding.  }
    \label{fig:sceneComplexity}
    \vspace{-1em}
\end{figure*}

\begin{table}[h]
\begin{center}
\scriptsize{
\tabcolsep=0.08cm
\begin{tabular}{ |c|c|c|c|c|c|c|c|c|c|c|c|c|c|c|c|c|c| }
\hline
\textbf{Method} & M-F &  Detect & C-F & MC-F &  I-Order ($\downarrow$)& RNN-Order ($\downarrow$)\\
\hline
\hline
PINET~\cite{DBLP:journals/corr/abs-2002-06604}  & 49.5 & 19.2 & - & 14.7 &   1.08 & - \\
\hline
PRNN & 52.9 & 40.5 & 24.5 &  45.9 &  0.894& -\\

Ours/PRNN & 51.7 & 53.1 & 49.9& 53.2 & 0.824 & - \\
\hline
TR \cite{Can_2021_ICCV}& 56.7  & 59.9 & 55.2 & 62.0 & 0.800& - \\

TR-RNN & 58.0 & 59.7 &  52.3 & 60.3 & 0.791 & 0.939 \\

\textbf{Ours/TR} &  \textbf{58.2} & \textbf{60.2} & \textbf{55.3} & \textbf{62.5} &  \textbf{0.776}& -  \\

\hline
\end{tabular}
}
\end{center}
\vspace{-2em}
\caption{Results on NuScenes. See Section~\ref{sec:metrices} for the metrics.}
\vspace{-0.5em}
\label{tab:nusc}
\end{table}

We report quantitative comparison with SOTA on the Nuscenes dataset in Tab.~\ref{tab:nusc}. The proposed formulation provides substantial boost in almost every metric for Polygon-RNN based methods. Compared to transformer-based methods proposed in \cite{Can_2021_ICCV}, our method performs better in all metrics. We also validate our method in Argoverse dataset in Tab.~\ref{tab:argo}. It can be seen that our method consistently outperforms the competitors. 

We also report the results of Polygon-RNN based method given the true initial centerline points in Tab.~\ref{tab:polygt}. Note that Polygon-RNN and Polygon-RNN(GT) are the same models with the same parameters, with the only difference being true or estimated initial points. The results show that our formulation is applicable in significantly different architectures and in different settings.

\begin{table}[h]
\begin{center}
\scriptsize{
\tabcolsep=0.08cm
\begin{tabular}{ |c|c|c|c|c|c|c|c|c|c|c|c|c|c|c|c|c|c| }
\hline
\textbf{Method} & M-F &  Detect & C-F & MC-F &  I-Order ($\downarrow$)& RNN-Order ($\downarrow$)\\
\hline
\hline
PINET~\cite{DBLP:journals/corr/abs-2002-06604}  & 47.2 & 15.1 & - & 24.5&   1.23 & - \\
\hline
PRNN & 45.1 & 40.2 & 31.3 & 42.8  &   1.09 & - \\

Ours/PRNN & 44.3 & 55.4 & 46.4& 49.5  & 1.05 & -  \\
\hline
TR\cite{Can_2021_ICCV} & 55.6 & 60.1 & 54.9 & 57.4  &   0.893 & - \\

TR-RNN & 55.8 & 53.6 & 54.4 &  52.2 &   0.953 & 0.951  \\

\textbf{Ours/TR} &  \textbf{57.1} & \textbf{64.2} & \textbf{58.1} & \textbf{58.5}  & \textbf{0.883}& -  \\

\hline
\end{tabular}
}
\end{center}
\vspace{-2em}
\caption{Results on Argoverse}
\vspace{-1em}
\label{tab:argo}
\end{table}

We provide evaluation of the proposed minimal cycle branch in Tab.~\ref{tab:mc_performance}. \textbf{H-GT-F} results in both datasets indicate that the transformer-based method is better in detecting true minimal cycles, hence estimating the true topology. Moreover, from \textbf{H-EST-F} results, it can be seen that the transformer-based method is more self-aware of the resulting road network estimate. Same conclusion can be drawn from the similarity between transformer's \textbf{H-GT-F} and \textbf{MC-F} values. This implies that the method outputs centerline estimates in consistence with its topological estimate. These results are expected since the transformer jointly estimates the centerlines and the minimal cycles while the Polygon-RNN output is staged.

\begin{table}[h]
\begin{center}
\scriptsize{
\tabcolsep=0.08cm
\begin{tabular}{ |c|c|c|c|c|c|c|c|c|c|c|c|c|c|c|c|c|c|c| }
\hline
\textbf{Dataset} & \textbf{Method} & M-F &  Detect & C-F & MC-F &  I-Order ($\downarrow$)\\
\hline

\multirow{2}{*}{\textbf{Nuscenes}} & PRNN(GT) & 71.1 & 76.4 &  52.9 &   66.9 & 0.645  \\
&Ours/PRNN(GT) & \textbf{72.6}& \textbf{77.2} & \textbf{55.0} &   \textbf{67.5}  & \textbf{0.642} \\
\hline
\multirow{2}{*}{\textbf{Argoverse}} & PRNN(GT) & 75.0 & 73.6 &  54.1 &   61.0 & \textbf{0.830} \\
& Ours/PRNN(GT) & \textbf{76.1} &\textbf{74.2}& \textbf{54.5} & \textbf{61.4}  &  0.844\\
\hline
\end{tabular}
}
\end{center}
\vspace{-2em}
\caption{Polygon-RNN with GT initial points results}
\vspace{-0.5em}
\label{tab:polygt}
\end{table}

An important observation is that the MC metrics show a clear correlation with I-Order, empirically proving the equivalence of MC covers and intersection orders. We observe that the TR-RNN method's direct order estimations are far from its real achieved edit distances. This indicates that recursive estimation of intersections is not as as accurate as our minimal cycle based formulation

The performance of different methods with increasing scene complexity in the Nuscenes dataset is reported in Fig.~\ref{fig:sceneComplexity}. As expected, the performance of all methods deteriorates with an increased number of centerlines, intersection points, and scene occlusions. Nevertheless, the proposed MC-based methods consistently produce better I-Order over the baselines. Some qualitative examples of the compared methods are shown in Fig.~\ref{fig:visual_compare}, where methods that use the MC branch are preferable again.

\begin{table}[h]
\begin{center}
\scriptsize{
\tabcolsep=0.08cm
\begin{tabular}{ |c|c|c|c|c| }
\hline
\textbf{Method} & \multicolumn{2}{|c|}{NuScenes} &  \multicolumn{2}{|c|}{Argoverse} \\
\hline
&  H-GT-F & H-EST-F &  H-GT-F & H-EST-F  \\
\hline
\hline

Ours/PRNN & 42.5 & 45.1 & 36.5 & 36.9  \\
Ours/PRNN(GT) &  51.0 & 55.9 & 44.8 & 49.3 \\
\textbf{Ours/TR} &   \textbf{60.9} & \textbf{73.0}  & \textbf{56.6} & \textbf{61.5} \\
\hline
\end{tabular}
}
\end{center}
\vspace{-2em}
\caption{ Performance of minimal cycle estimation head.}
\vspace{-0.5em}
\label{tab:mc_performance}
\end{table}

\section{Conclusion}
We studied local road network extraction from a single onboard camera image. To encourage topological consistency, we formulated the minimal cycle matching strategy by means of matching only their covers. Our formulation is then used to derive losses, to train neural networks of two different architectures, namely Transformer and Poly-RNN. Both architectures demonstrated the importance of the proposed MC branch, and thus the formulated loss function, on two commonly used benchmark datasets. The proposed formulation, and the method, have the potential to be used in many other computer vision problems which require topologically consistent outputs, for example, indoor room layout estimation or scene parsing.   


\noindent\textbf{Limitations.}
The theoretical assumptions are mild for most modern road networks. The extraction of minimal cycles for training is time consuming and should be done offline.

{\small
\bibliographystyle{ieee_fullname}
\bibliography{egbib}
}

\newpage

\section{Appendix}

\section{Code}
The github page: \url{https://github.com/ybarancan/TopologicalLaneGraph}

\section{Theory}

\subsection{Assumptions}

In this subsection, we provide the additional assumptions we make on top of the ones listed in the main paper.


\paragraph{A curve cannot appear more than once in a minimal cycle.} This holds because the divergence and convergence of the lanes result in centerlines with different identities. 

\paragraph{Direction assumption.} Let the direction of a minimal cycle $MC_i$ be defined as $D(MC_i) \in \{0,1\}$, i.e. 0=clockwise and 1= counter-clockwise. It induces an ordering of the segments in the minimal cycle relative to an intersection point $P$ in the cycle such that $u_i > u_j$ for any two segment $u_i$ and $u_j$ in the minimal cycle if $u_i$ appears later than $u_j$ when traversing the cycle starting from $P$ in the direction $D(MC_i)$. Similarly an intersection point $p_i > p_j$ if $p_i$ appears later than $p_j$. 

Consider two minimal cycles $MC_i$ and $MC_j$ where at least one curve $U_n$ appears in both cycles. The segment of $U_n$ in $MC_i$ is $u_n$ and the segment in $MC_j$ is $u_n'$. Since a curve is continuous there is a path connecting $u_n$ to $u_n'$ such that it only consists of $U_n$. Let this path be $VG(u_n) = VG(u_n')$ and the intersection point that this path intersects $MC_i$ be $PP(u_n)$ and the point it intersects in $MC_j$ be $PP(u_n')$. The other intersection points that define their respective segments is $K(u_n)$ and $K(u_n')$ respectively. 

We define that two minimal cycles are neighbors if they share at least one intersection point. Based on this, we have the following property (direction property): Let $P=PP(u_n)$ and $D(MC_i)$ be the direction from $PP(u_n)$ to $K(u_n)$ through $u_n$. Then $D(MC_j) \neq D(MC_i)$ and there exists a path from $K(u_n)$ to $PP(u_n')$ entirely on two minimal cycles and following the direction of the minimal cycle it is in such that it does not intersect $U_n$ except at $K(u_n)$ and $PP(u_n')$. If $MC_i$ and $MC_j$ are not neighbors, then the path can teleport to $MC_j$ at any intersection point other than $PP(u_n')$ or $K(u_n')$.

For minimal cycles with shared segments or no neighborhood at all, this means: $\; \exists \; (D(MC_i), D(MC_j)) \;|\; D(MC_i)\neq D(MC_j) \;\& \; \forall \;PP(u_n) \in \{U_n \cap MC_i\}, \forall K(u_n) \in \{U_n \cap MC_i\}\;;\; PP(u_n) \mathop{\lessgtr} K(u_n) \implies PP(u_n') \mathop{\lessgtr} K(_n')$. This means there exists a tuple of opposing directions such that for all the shared curves in both minimal cycles, if the port point of the curve in a cycle appears later/earlier than its corresponding non-port point, the same order has to be preserved in the other cycle. Note that, this holds for all reference points that the directions are defined on. Moreover, negation of both directions does not affect the condition. Therefore, for any arbitrarily assigned $D(MC_i)$, $D(MC_j)$ such that $D(MC_i) \neq D(MC_j)$ has to support this condition. In Fig \ref{fig:topo-exp} subfigure i-a, we focus on shared curve $C_1$ in minimal cycles A and C. These minimal cycles are not neighbors. The port point of $C_1$ in A, $PP(c_1)$, is the blue dot and the port point in C, $PP(c_1')$, is the red dot. Thus, the path should be from green dot , $K(c_1)$, to red dot, $PP(c_1')$. The path is shown with green arrows while the directions of the minimal cycles are shown by the black arrows surrounding the letters.

If two minimal cycles share intersection points but no segments, the condition is trivially satisfied for the curves at these intersection points since $PP(u_n')=PP(u_n)$ and the path is simply either of the minimal cycles. In Fig \ref{fig:topo-exp} subfigure i-b, this case is shown. The port point of $C_3$ in D, $PP(c_3)$, is the same point as $PP(c_3')$ and represented by the red dot. While the green dots represent the $K(c_3')$ in E and $K(c_3)$ in D. The green arrows in D show the path where the beginning point is $K(c_3)$. The arrow in E shows the path when the beginning point is $K(c_3')$. In both cases, the entire paths are in either of the minimal cycles, hence obeying the assumption. 

This assumption is suitable for our case since the lanes tend to have limited curvature. Therefore, for a minimal cycle that violates this assumption, the curves need to have significantly varying curvatures. An example with two minimal cycles with same minimal covers violating direction assumption is given in Fig.~\ref{fig:neighbor-counter}. 

\begin{figure*}
    \centering
    \includegraphics[width=.8\linewidth]{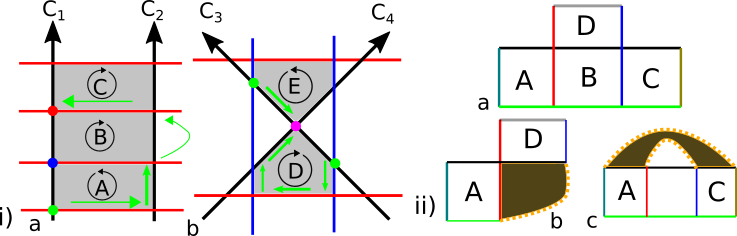}
    \caption{i) The direction property for 2 most common types of minimal cycles. a) We consider cycles A and C with focus on curve $C_1$. Green dot represents $K(C_1)$, red dot represents $PP(C_1')$ and blue dot is $PP(C_1)$. The green arrows show the path from green dot to red dot, teleporting from A to C over B. In b), we show how the shared intersection point trivially satisfies the property. We focus on $C_3$ where green dots show $K(C_3)$ and $K(C_3')$. In this example, $PP(C_3')$ = $PP(C_3)$, thus red and blue dots coincide and shown in purple. ii) Accessibility property and the resulting areas (shaded) for 2 pairs of minimal cycles (A-D) and (A-C). For A-D, the shared intersection point acts as the first artificial path.}
    \label{fig:topo-exp}
    
\end{figure*}

\paragraph{Accessibility assumption.} This assumption states that for all pairs of minimal cycles, it is possible to draw two artificial curves such that the curves start at different intersection points on one cycle and end on different intersection points on the other cycle. The curves cannot intersect and the resulting area bounded by the two curves and two minimal cycles is not intersected by any other curves that are in both of the minimal cycles except on the cycles. In other words, pick two minimal cycles and remove any curve that does not appear in any of those two cycles. Then, if a curve appears in only one of the cycles, only keep the segment of the curve that is in one of the considered minimal cycles. If a curve appears in both cycles, keep its segments in both cycles as well as its $VG$. In the remaining structure it is possible to draw two paths from two distinct intersection points on one cycle to two distinct intersection points on the other cycle such that these two paths do not intersect and no curve intersects the area bounded by the these two artificial curves and the minimal cycles. This property is trivially satisfied by the minimal cycles that share a segment since the two endpoints of a shared segment serve as two distinct points on both cycles. 

We show some example cases for this assumption in Fig \ref{fig:topo-exp} subfigure ii. In a) we see the complete lane graph. In b) we show the mentioned area with shaded region when the minmal cycles A and D are selected. Since these cycles share an intersection point, it can act as one of the two paths needed to enclose the area. The other path is from green curve to blue curve. In c), we see the case if we focus on minimal cycles A and C. Since the gray curve is not in any of the minimal cycles in focus, we simply ignore it. Moreover, since red and blue curves only appear in one of the cycles, we ignore their segments that are not in the minimal cycles A and C. However, black and green curves appear in both cycles, thus all of their segments are included. The proposed area is the shaded region that satisfies the assumption. 

This assumption is suitable for our case since merging or divergence of lanes are represented with different road segments that correspond to different curves in our formulation. This assumption's underlying intuition is that curves are short and limited in their change in curvature. 
.

\begin{assumption}
Any two curves can intersect at most once.
\end{assumption}
\begin{assumption}
No curve can intersect with itself.
\end{assumption}
\begin{assumption}
No curve is floating, i.e. every curve is \textbf{connected} (hence intersecting) with another curve in its start and endpoints. 
\end{assumption}
\begin{assumption}
Curves can appear in one minimal cycle at most once.
\end{assumption}
\begin{assumption}
Direction assumption
\end{assumption}
\begin{assumption}
Accessibility assumption
\end{assumption}

\begin{figure}
    \centering
    \includegraphics[width=.7\linewidth]{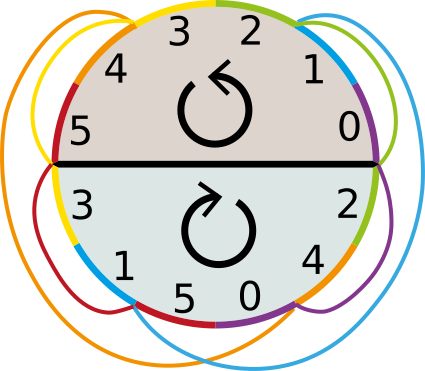}
    \caption{The \textbf{simplest} counter example with shared segments if direction assumption is violated. Obviously, it is very unlikely that such a structure would occur in a lane graph. The orange curve (curve 4) and blue curve (curve 1) violate the assumption.}
    \label{fig:neighbor-counter}
    
\end{figure}
\begin{lemma}
\label{lemma1}
A minimal closed polycurve (minimal cycle) $MC$ is uniquely identified by its minimal cover $B$
\end{lemma}
\begin{proof}
Assume there is more than one minimal cycles defined by the same minimal cover and let the first minimal cycle be $MC_1$ with $D(MC_1)$ and the second $MC_2$ with $D(MC_2)$. There are two scenarios: the cycles share at least one segment or not. Let's start with the former case.  

For proving the shared segment case of the lemma, let $D(MC_1) \neq D(MC_2)$. Pick the intersection point where the last shared segment ends (end defined by $D(MC_1)$) as the reference point. Segments $u_0, ..., u_N$ refer to the segments other than the shared ones in $MC_1$ with the order induced by the reference point and $D(MC_1)$. The curves these segments belong to are $U_0, ..., U_N$. Similarly, $w_0, ..., w_M$ are the segments in $MC_2$ with the same reference point but $D(MC_2)$ with the curves $W_0, ..., W_M$. Note that this means $u_0$ and $w_0$ are neighbors as well as $u_N$ and $w_M$. Assume $W_0 \neq U_0$ and $W_j=U_0$ where $j>1$ since $W_0$ and $U_0$ already intersect. We begin by deciding whether $PP(u_0) > K(u_0)$ or not. Let $PP(u_0) > K(u_0)$. From the \textbf{Assumption 5}, if $VG(u_0)$ follows $D(MC_1)$, $VG(u_0)$ covers the space around $u_0,... u_N, w_M, ... u_0'$ and otherwise  $u_0, w_0, ... u_0'$. That is, since $PP(u_0) > K(u_0)$, $PP(u_0') > K(u_0')$. In both cases, $VG(w_0)$ has to intersect with $VG(u_0)$ to connect $W_0$ with $MC_1$. This is true because $MC_1$ and $MC_2$ are minimal cycles and thus, no path can pass through these cycles. See Fig ~\ref{fig:proof} for visualization. Since $W_0$ and $U_0$ are neighbors, intersection of $VG(w_0)$ and $VG(u_0)$ would mean $W_0$ and $U_0$ intersecting twice, thus violating \textbf{Assumption 1}. Therefore, $PP(u_0) < K(u_0)$. In this case, $PP(u_0) = PP(w_0)$ to avoid intersection twice. This means they share the port points. Moreover, since a curve cannot appear more than once in a minimal cycle (\textbf{Assumption 4}), we know that $N=M$.

\textbf{Statement 1: } Consider two segments $u_i$ and $u_{i+1}$ in $MC_1$. If they do not share a port point, $u_{i+1}' > u_i'$. The reason is  $VG(u_i)$ either covers $[u_i, ... , u_0, w_0, ..., u_i']$ or $[u_i, ... , u_{|U|-1}, w_{|W|-1}, ..., u_i']$ based on \textbf{Assumption 5}. In either case, $VG(u_{i+1})$ cannot reach $[u_i, ... , w_0, ..., u_i']$, without intersecting $VG(u_i)$ since the area in $MC_1$ and $MC_2$ cannot be intersected because they are minimal cycles. Note that this holds if $MC_1$ and $MC_2$ share at least one segment.

Now consider $u_1$; We know that $u_0$ and $u_1$ do not a share port point because $u_0$ already has a port point with its intersection with $w_0$. Having two port points would mean $U_0$ itself is a cycle which is not allowed by \textbf{Assumption 2}. From Statement 1, we know that $u_1' > u_0'$. If $PP(u_1) > K(u_1)$, $PP(u_1') > K(u_1')$. Therefore, $u_1'-1$ (the neighbor of $u_1'$ that appears just before $u_1'$) cannot share a port point with $u_1'$. This means from Statement 1, $(u_1'-1)'$ has to appear before $u_1$. However, only $u_0$ appears before $u_1$ and $U_0$ and $U_1$ cannot be neighbors in $MC_2$ since they are already neighbors in $MC_1$. Therefore, we reach the conclusion that $PP(u_1) < K(u_1)$. 

Since $PP(u_1) < K(u_1)$, $U_2$ has to appear after $U_1$. With a similar argument as $U_1$, $PP(u_2) > K(u_2)$ is impossible because it would require its neighbor with a lower order in $MC_2$ to either be $U_0$ or $U_1$ to avoid intersecting $VG(u_2)$ twice. This means $PP(u_2) < K(u_2)$. Iteratively applying this argument, we reach that $u_N' > u_{N-1}' > ... > u_1' > u_0'$. This means the same order has to be preserved which indicates neighbors in $MC_1$ are also neighbors in $MC_2$ violating the rule of at most one intersection between any two curves.

Now let $W_0 = U_0$. The same argument carries over. Simply remove $U_0$ from $U$ and $W_0$ from $W$ and apply the same procedure as the case $W_0 \neq U_0$.

\paragraph{No shared segment.} We know that there exists an area between two cycles such that no shared curves intersect from \textbf{Assumption 6}. Since $MC_1$ and $MC_2$ have the same list of curves, all curves are shared. Therefore, no curve in either $MC_1$ or $MC_2$ intersects this area. This implies, $MC_1$ and $MC_2$ can be essentially considered as neighbors. We can apply the same arguments as used in case with shared segment. For this, consider the area that no curve can intersect \textbf{(Assumption 6)} and the part of $MC_1$ and $MC_2$ it intersects. Let the curve it intersects on $MC_1$ be $U_0$ and correspondingly $W_0$ in $MC_2$. Moreover, let the intersection region of this artificial curve and $MC_1$ be an artificial segment $a$ and correspondingly $a'$ in $MC_2$ with this artificial curve being $A$. 
We know $PP(u_0) \ngtr a$ from the shared segment case. The reason being $u_0'-1$ would be completely cut off from $MC_1$ by $VG(U_0)$ and $A$. Thus, $PP(u_0) < a$. Similar to the shared segment case, this means $u_1'$ appears after $u_0'$ in $MC_2$. Moreover, $PP(u_1) > K(u_1)$ would imply that $u_1'-1$ would be completely cut-off from $MC_1$. Therefore $PP(u_1) < K(u_1)$, which implies $u_2'$ appears after $u_1'$. The same iterative argument shows that $u_N' > u_{N-1}' > ... > u_1' > u_0'$.     


\end{proof}

\begin{figure}
    \centering
    \begin{subfigure}[b]{.7\linewidth}
    \includegraphics[width=.7\linewidth]{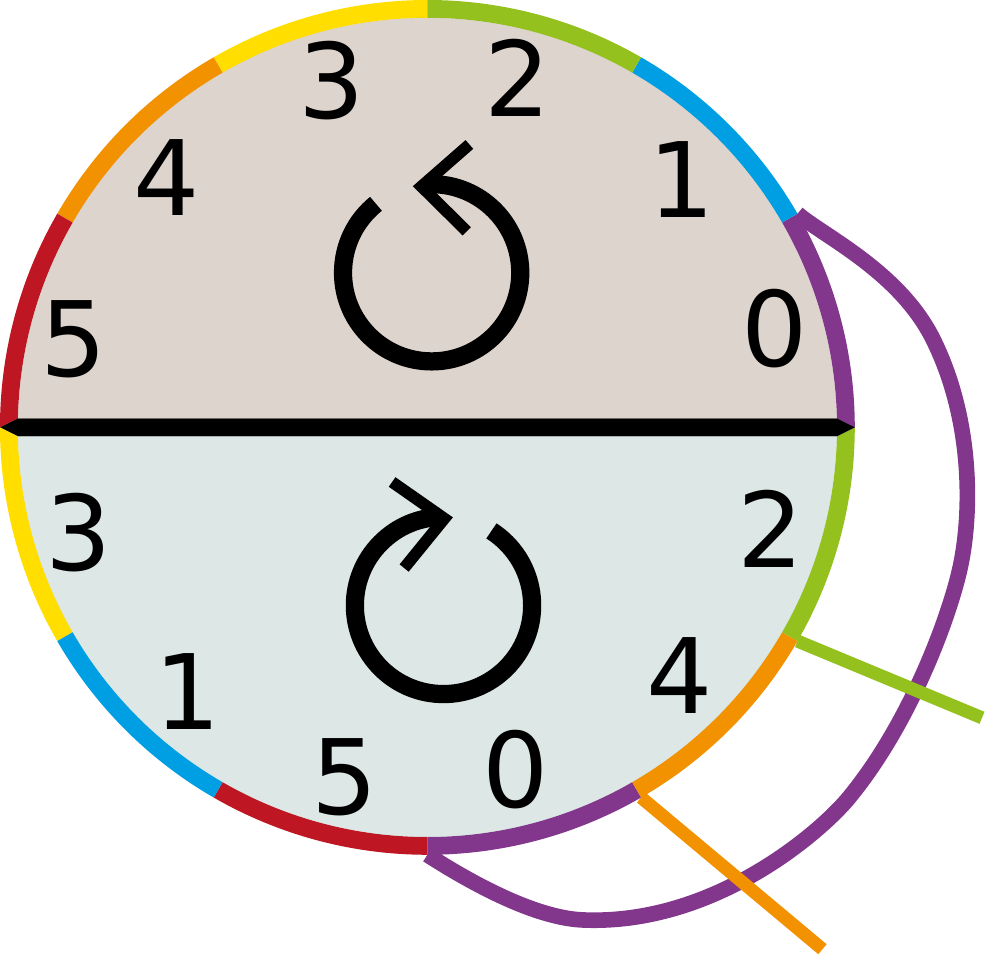}
    \end{subfigure}
    \begin{subfigure}[b]{.7\linewidth}
    \includegraphics[width=.7\linewidth]{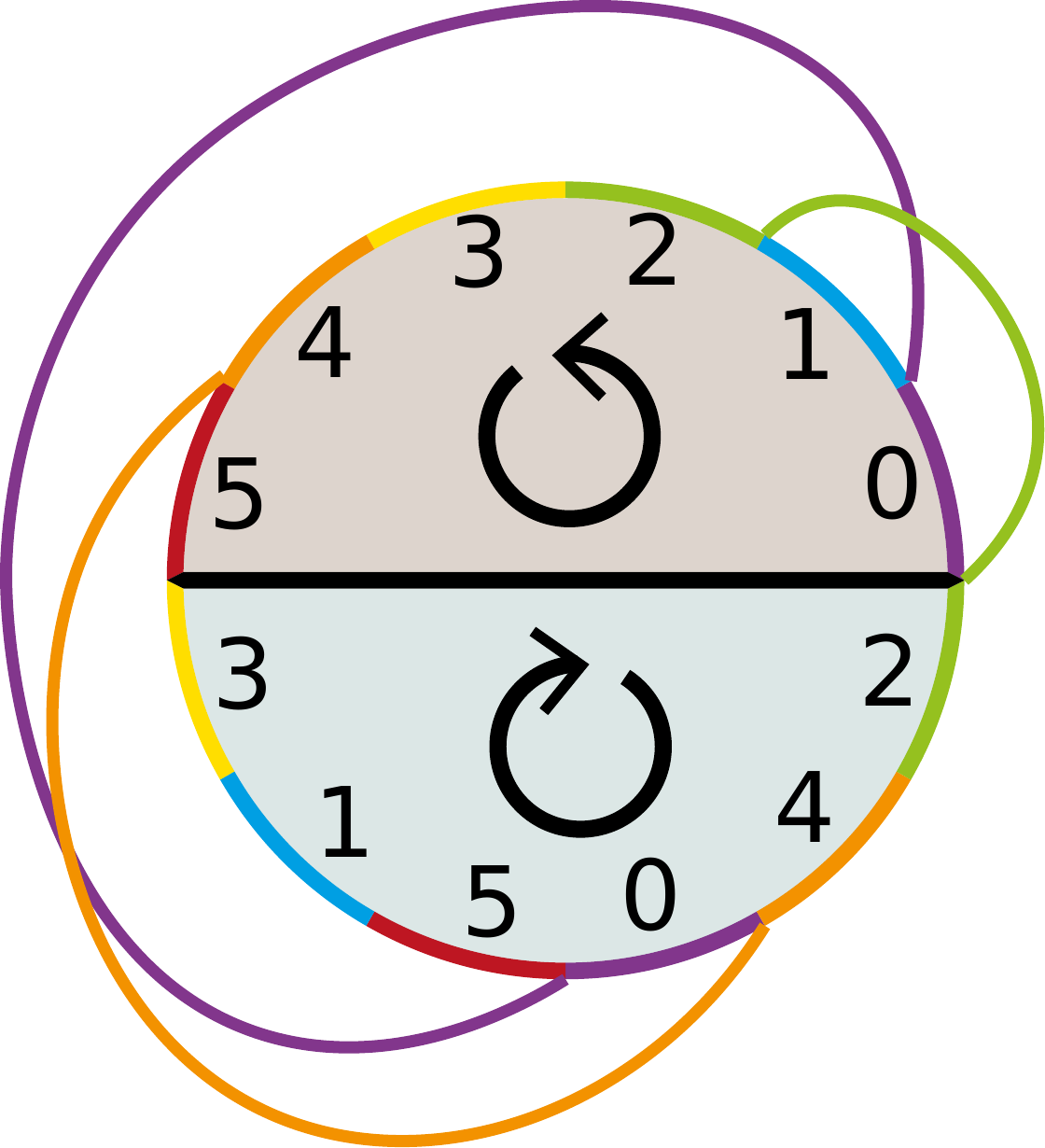}
    \end{subfigure}
    
       \caption{The illustration of the proof for Lemma 1.1. Due to the direction assumption, curve 0 blocks curve 4 and curve 2 accessing the upper minimal cycle ($MC_1$ in the proof). This situation happens regardless of the route $VG(U_0)$ takes (here it is the purple path connecting the curve 0 between two minimal cycles).}
       
       \label{fig:proof}
\end{figure}

\begin{lemma}
\label{lemma2}
Let a set of curves $C_1$ and the induced intersection orders $I_1$ form the structure $T_1=(C_2, I_2)$. Applying any deformations on the curves in $C_1$ (but no removal or addition of curves) and the resulting induced intersection orders create $T_2=(C_2, I_2)$. $I_1 = I_2 \iff {MC}_1 == {MC}_2$. In other words, the global intersection orders of the two structures are same if and only if the set of minimal cycles are same.
\end{lemma}

\begin{proof}
Let's begin by proving the forward statement. Consider a $MC$ in $G_1$ and the set of intersection points and the curves creating this $MC$. Since $G_2$ has the same $I$ as $G_1$, a closed polycurve $CC$ in $G_2$ can be formed by the same set of intersection points and curves. Assume some curve $C_p$ intersects the area of closed polycurve $CC$ and hence $CC$ is not a minimal cycle. We know that no curve is floating from \textbf{Assumption 3} so $C_p$ has to intersect some other curve either in $CC$ or on $CC$. Either way, this means alteration in $I$ of at least 2 curves. 

For the opposite direction, consider a pair of identical $MC$s (one in $G_1$ and one in $G_2$). The cycles are formed by the same curve segments by definition. The curve segments are, in turn, defined by the intersection points on their respective curves. Since we know no curve is removed or added, the identical minimal cycles in the structures are formed by the identical intersection points. Since minimal cycles create a partition of the space, each curve segment has to appear in at least one minimal cycle. Thus, the set of intersection points $I_1 = I_2$.

\end{proof}

\section{Architectures}
Here, we provide some details regarding the architectures. 
\subsection{MC-Polygon-RNN (Ours/PRNN) }

Polygon-RNN produces $2N$ feature maps where $N$ is the number of initial points (which is the same as the number of centerlines). Since we use three control points, Polygon-RNN does two iterations to produce the rest of the control points. The $N$ final feature maps are then passed through an MLP to produce $N$ feature vectors. In the transformer NLP setting, these vectors correspond to the transformer encoder output which is the processed input sequence. Thus, in the transformer decoder, cycle queries attend to all the centerlines to produce the output.  

\begin{figure*}
    \centering
    \includegraphics[width=\linewidth]{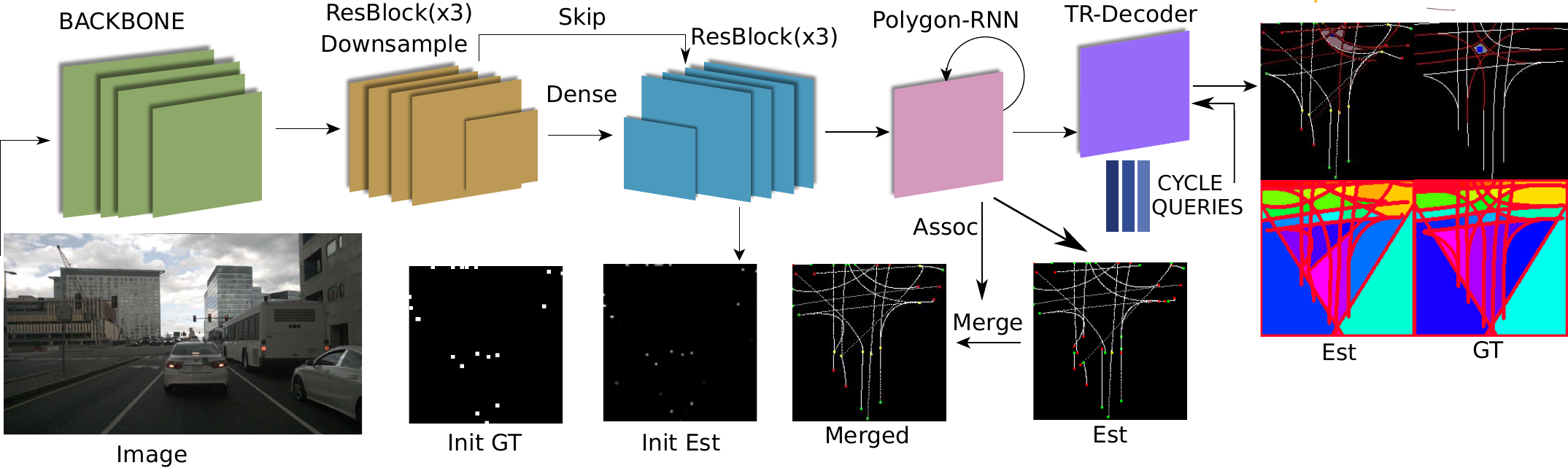}
    \caption{MC-Poly-RNN architecture is formed by addition of a transformer decoder on top of Polygon-RNN's last feature map outputs. }
    \label{fig:supp-polyrnn}
    
\end{figure*}

\subsection{TR-RNN }

In order to investigate the feasibility of estimating the order of intersections directly, we opted for an architecture that combines transformers and RNNs. Specifically, on top of the base transformer model, we use an RNN. The input to the RNN is $V_c$, i.e. the processed query vectors. Each processed query vector is processed by the RNN independently. Let the RNN input vector be $V_c(i)$, i.e. $i$th processed query vector. We refer to it as the reference curve. At each time step $t$ of the RNN, the output is a probability distribution over the estimated curves and the boundary curves, as well as an end token. Therefore, at each time step $t$, the output is $N+K+1$ dimensional, with $N$ estimated curves, $K$ boundary curves and one end token. Whenever the RNN outputs the end token, we gather the estimates at the previous steps. These outputs form the ordered sequence of intersections for the reference curve. Since the query vectors are processed jointly by the transformer, $V_c$ carries information about the whereabouts of the other curves as well. We measure the performance of this RNN by using the same order metric 'I-Order' on the RNN estimates.   

\begin{figure*}
    \centering
    \includegraphics[width=.8\linewidth]{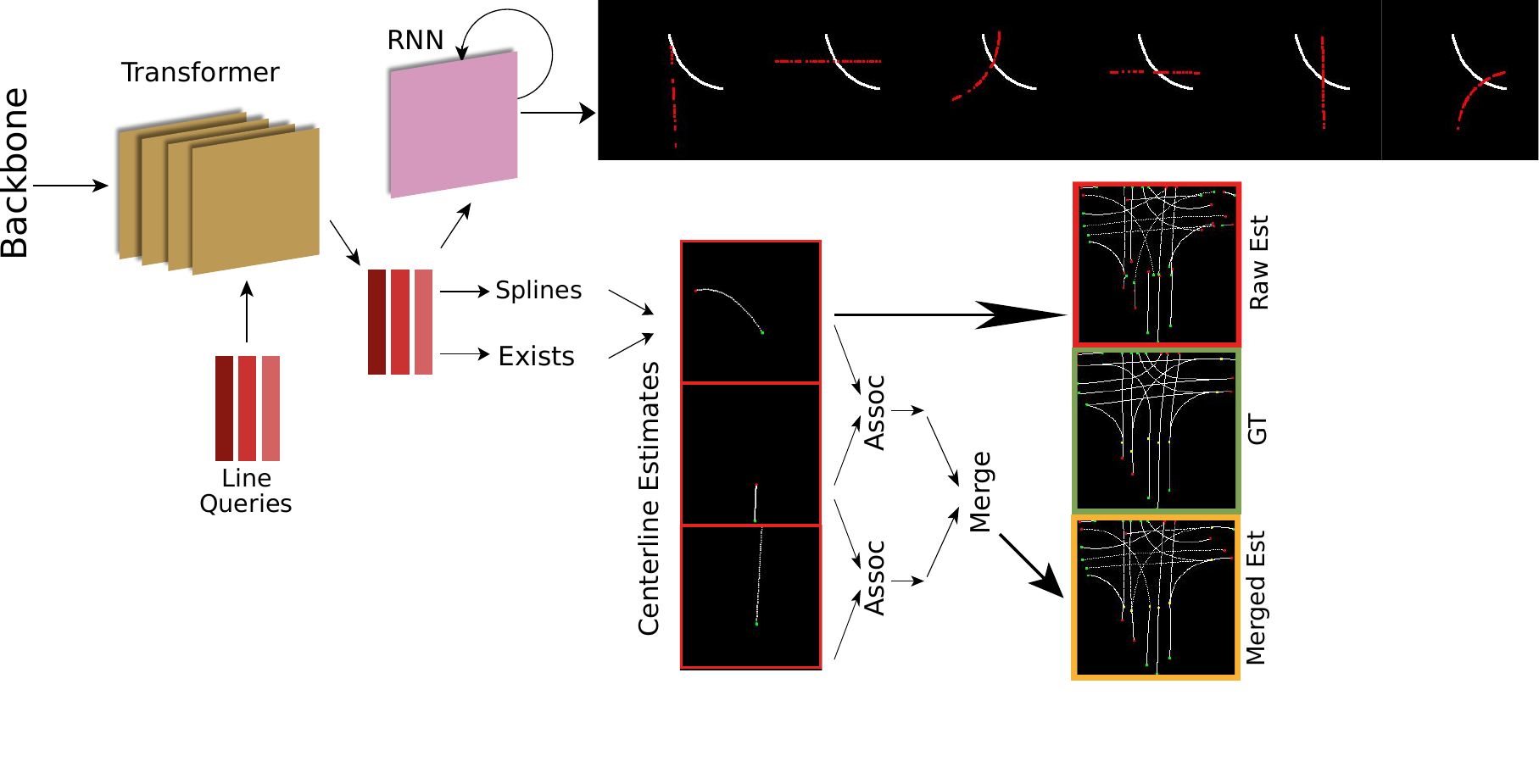}
    \caption{TR-RNN has an RNN operating on the transformer outputs. For each transformer processed curve query vector, it outputs a sequence of fixed sized ($N+K$ dimensional) vectors. Each such vector indicate the probability distribution over the estimated curves that the current reference query intersects in that order. In the figure, we show the reference query with white curve and the selected intersecting curve at each time step with red. This sequence is a direct estimation of the intersection order.}
    \label{fig:supp-polyrnn}
    
\end{figure*}

\section{Metrics}

We introduced 2 new metrics that measure topological structure accuracy. Here we present additional explanations on the calculation and the justifications for these metrics.

\subsection{Minimal Cycle Minimal Cover }

In order measure the performance of the methods on estimating the minimal cycles, we introduce a precision-recall based metric. The metric operates on min-matched set of curves. Similar to the connectivity metric, let $R(i)$ be the index of the target that the $i$th estimation is matched to and $S(n)$ be the set of indices of estimations that are matched to target $n$. Now let the GT minimal cycles be $MC_{GT} \in [0,1]^{M'\times {N'+K}}$ while $MC_{Est} \in [0,1]^{M\times {N+K}}$. We will form true positive (TP), false positive (FP) and false negative (FN) matrices of size $M'\times M \times {N'+K}$.

TP(i, j, k) = 1 if $ MC_{GT}(i,k) = 1 \; \& \; \exists \;n\; |\; ((R(n) = k) \;\&\; (MC_{Est}(j, n) = 1))$. In words, TP(i, j, k) is 1 if $i$th true minimal cycle includes $k$th true curve and $j$th estimated cycle has a curve in it that is matched to $k$th true curve in matching.

FN(i, j, k) = 1 if $ MC_{GT}(i,k) = 1 \; \& \; \nexists \;n\; |\; ((R(n) = k) \;\&\; (MC_{Est}(j, n) = 1))$. 

FP(i, j, k) = 1 if $ MC_{GT}(i,k) = 0 \; \& \; \exists \;n\; |\; ((R(n) = k) \;\&\; (MC_{Est}(j, n) = 1))$. 

We sum up the TP, FP and FN matrices along the last dimension to obtain TP', FP' and FN'. Then $H\in \mathbb{R}^{M'\times M} = FN' + FP'$. We run Hungarian matching on this matrix and we simply select the resulting indices from TP', FP' and FN' and take the sum to get the statistics. If $M' > M$, we consider all the positive entries in unmatched true minimal cycles as false negative. 

\paragraph{H-GT-F and H-EST-F}. MC-F measures the minimal cycle accuracy of the resulting lane graph. However, we also want to measure the performance of the minimal cycle detection by the sub-network. Therefore, we apply the same procedure described above on the estimates of the MC detection sub-network. \textbf{H-GT-F} measures the F-score of the detection subnetwork's estimates compared to the true minimal cycles. This uses the same procedure as MC-F. Therefore, this measures how good the detection head is in detecting the minimal cycles that \textbf{should be} created in the estimated lane graph. \textbf{H-EST-F} measures the minimal cycle detection network's performance in detecting the minimal cycles that are induced by the estimated lane graph. Thus, it shows how good the detection network is in detecting the minimal cycles \textbf{that are created} in the estimated lane graph. In summary, \textbf{MC-F} is a similarity measure between induced lane graph and true lane graph, \textbf{H-GT-F} is between detection estimates and true lane graph and \textbf{H-EST-F} is between detection estimates and induced lane graph. Together, these metrics give a full picture of the performance of the detection network as well as the whole network in estimating the lane graph.

\subsection{Intersection Order}\label{sec:interOrder}

In order to measure the intersection order, we get the best match for each true curve. Let $M(i)$ be the index of the target that the $i$th estimation is matched to and $S(n)$ be the set of indices of estimations that are matched to target $n$. Then the best match curve for a true curve n is $R_n = \argmin_{s} \text{L1}(C_n, s), s\in S(n)$. For the set of true curves ${C_n}$ with $|S(n)| > 0$, given the matched pairs, we extract the intersection orders. We similarly extract the intersection orders for the estimated curves that are the best match for some true curve. Given the set of matched intersection orders we run Levenshtein edit distance and normalize it by the length of intersection order sequence of the true curve. For the unnmatched true curves, we consider the distance to be 2 (the normalized distance that would result from removing all the elements in the estimated sequence and adding the elements in the true sequence if the estimated sequence is the same length as true sequence). The I-Order of a frame is then the mean over the normalized edit distances.

\paragraph{RNN-Order.} The procedure explained above in Section~\ref{sec:interOrder} is applied to the outputs of TR-RNN to measure how accurate the RNN is in estimating the intersection orders. Instead of the orders extracted from the estimated lane graph, we simply use the RNN outputs.

\section{Results}

We provided detailed statistical results in the main paper. Here, we provide more visual results. Visuals confirm the statistical findings that the proposed formulation provides improvement in preservation of the topological structure of the true road network.

\begin{figure*}
    \centering
    \includegraphics[width=.9\linewidth]{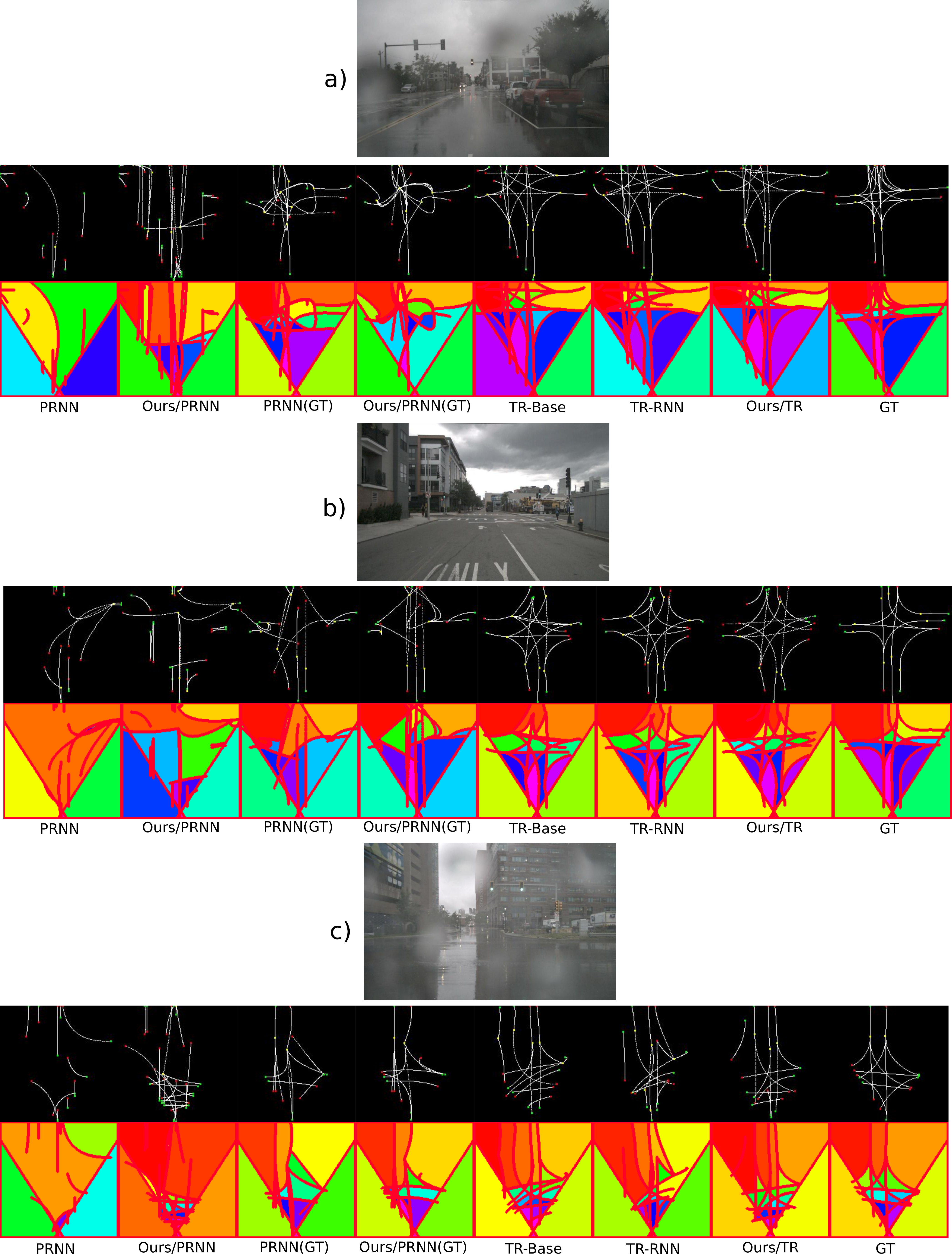}
    \caption{For 3 samples image (above) and the lane graphs and induced minimal cycles (below) in Nuscenes dataset.}
    \label{fig:visuals1}
    
\end{figure*}

\begin{figure*}
    \centering
    \includegraphics[width=.9\linewidth]{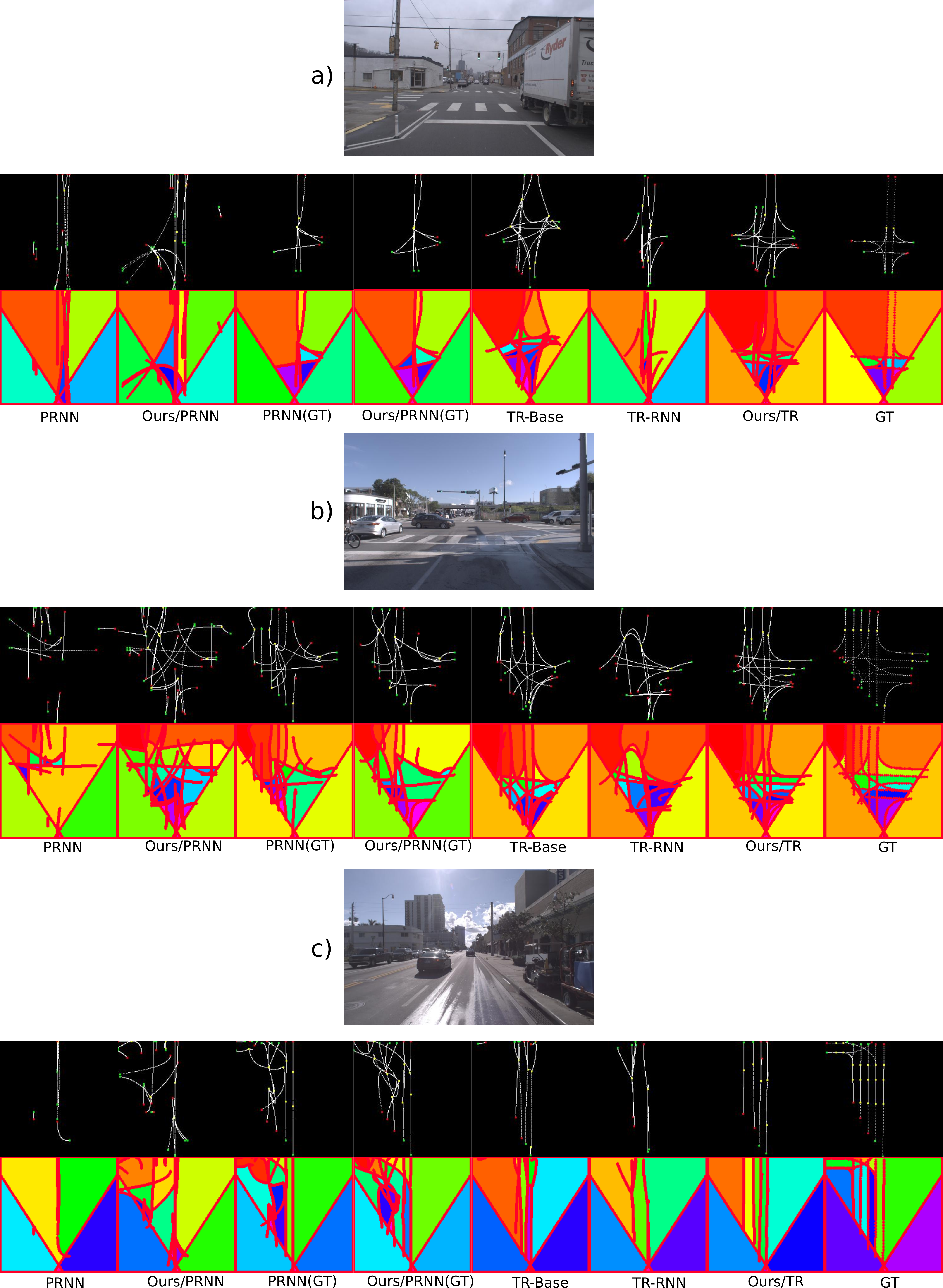}
    \caption{For 3 samples image (above) and the lane graphs and induced minimal cycles (below) in Argoverse dataset.}
    \label{fig:visuals1}
    
\end{figure*}

\end{document}